\documentclass{article}

\usepackage[preprint]{neurips_2025}

\usepackage[utf8]{inputenc} 
\usepackage[T1]{fontenc}    
\usepackage{hyperref}       
\usepackage{url}            
\usepackage{booktabs}       
\usepackage{amsfonts}       
\usepackage{nicefrac}       
\usepackage{microtype}      
\usepackage{xcolor}         
\usepackage{microtype}
\usepackage{graphicx}
\usepackage{subcaption}
\usepackage[normalem]{ulem}
\usepackage{amsmath}
\usepackage{amssymb}
\usepackage{mathtools}
\usepackage{amsthm}
\usepackage{bm}
\usepackage{cancel}
\usepackage{pgfplots}
\usepackage{multirow}
\usepackage{comment}

\theoremstyle{plain}
\newtheorem{theorem}{Theorem}[section]
\newtheorem{proposition}[theorem]{Proposition}
\newtheorem{lemma}[theorem]{Lemma}
\newtheorem{corollary}[theorem]{Corollary}
\theoremstyle{definition}
\newtheorem{definition}[theorem]{Definition}

\theoremstyle{remark}
\newtheorem{remark}[theorem]{Remark}

\newcommand{\relu}{\text{ReLU}}

\title{Leveraging KANs for Expedient Training of Multichannel MLPs via Preconditioning and Geometric Refinement}

\author{%
  Jonas A. ~Actor
    \\
  Center for Computing Research\\
  Sandia National Laboratories\\
  Albuquerque, NM 87123 \\
  \texttt{jaactor@sandia.gov} \\
  \And
  Graham Harper \\
  Center for Computing Research\\
  Sandia National Laboratories\\
  Albuquerque, NM 87123 \\
  \AND
  Ben Southworth \\
  Los Alamos National Laboratories \\
  Los Alamos, NM 87545 \\
  \And
  Eric C. Cyr \\
  Center for Computing Research \\
  Sandia National Laboratories \\
  Albuquerque, NM 87123 \\
}

\begin{document}

\maketitle

\begin{abstract}
Multilayer perceptrons (MLPs) are a workhorse machine learning architecture, used in a variety of modern deep learning frameworks. However, recently Kolmogorov-Arnold Networks (KANs) have become increasingly popular due to their success on a range of problems, particularly for scientific machine learning tasks. 
In this paper, we exploit the relationship between KANs and multichannel MLPs to gain structural insight into how to train MLPs faster. 
We demonstrate the KAN basis (1) provides geometric localized support, and (2) acts as a preconditioned descent in the ReLU basis, overall resulting in expedited training and improved accuracy.
Our results show the equivalence between free-knot spline KAN architectures, and a class of MLPs that are refined geometrically along the channel dimension of each weight tensor. We exploit this structural equivalence to define a hierarchical refinement scheme that dramatically accelerates training of the multi-channel MLP architecture. We show further accuracy improvements can be had by allowing the $1$D locations of the spline knots to be trained simultaneously with the weights. These advances are
demonstrated on a range of benchmark examples for regression and scientific machine learning.
\end{abstract}

\section{Introduction}
\label{Introduction}
Multilayer perceptrons (MLPs) \citep{mcculloch1943logical, rosenblatt1958perceptron, pinkus1999approximation} are a classical deep learning architecture, which exploit the composition of affine maps with a nonlinear scalar activation function; MLP architectures or blocks appear in many state-of-the-art applications, including (but not limited to) variational autoencoders \citep{kingma2013auto} and transformers \cite{vaswani2017attention}.
An alternative architecture,
Kolmogorov-Arnold Networks (KANs) \citep{liu2024kan}, have gained tremendous popularity in recent literature, having successfully been used for a range of tasks, including computer vision \citep{cang2024can, cacciatore2024preliminary, cheon2024demonstrating}, time series analysis \citep{vaca2024kolmogorov, dong2024kolmogorov}, scientific machine learning \citep{abueidda2024deepokan, toscano2024pinns, wu2024kolmogorov, koenig2024kan, jacob2024spikans, patra2024physics, howard2024finite, rigas2024adaptive}, graph analysis \citep{kiamari2024gkan, bresson2024kagnns}, and beyond; see \citep{somvanshi2024survey, hou2024comprehensive} and references therein for a list of versions, extensions, and applications of KANs. 

The structural similarity between MLPs and KANs has been noted multiple times in the literature
\citep{pinkus1999approximation, igelnik2003kolmogorov, leni2013kolmogorov, he2024mlp, qiu2025powermlp}.  Both invoke approximation properties with arguments that expand upon the Kolmogorov Superposition Theorem (KST) \citep{kolmogorov1957representation}, also called the Kolmogorov-Arnold Superposition Theorem, which states there exist functions $\phi_{pq} \in C([0,1])$, for $p=1,\dots,n$ and $q=1,\dots,2n+1$, such that for any function $f \in C([0,1]^n)$, there are functions $\varphi_{q} \in C(\mathbb{R})$ such that 
\begin{equation}
f(x_1,\dots,x_n) = \sum_{q=1}^{2n+1} \varphi_{q}\left( \sum_{p=1}^n \phi_{pq}(x_p) \right).
\end{equation}
KANs, inspired by this statement, replace the functions $\varphi_{q}$ and $\phi_{pq}$, which are classically not differentiable at a dense set of points and only Holder continuous \citep{braun2009constructive, sprecher1993universal, actor2018computation}, with learned functions expressed in a spline basis \citep{liu2024kan}. Function composition is used to increase the depth of the resulting representation: for a layer $\ell$ with $P$ inputs and $Q$ outputs define
\begin{equation}
    \label{eq:kan_layer}
 x^{(\ell+1)}_q = \sum_{p=1}^P \phi^{(\ell)}_{pq}( x^{(\ell)}_p ) \mbox{ for } q=1,\dots,Q,
\end{equation}
with $\phi^{(\ell)}_{pq}$ trained from data. MLPs are similarly architectural cousins to the restatement of KST provided by \citep{lorentz1962metric, sprecher1993universal, laczkovich2021superposition} where the same statement for $f \in C([0,1]^n)$ can be made using a single inner function $\phi$ and single outer function (i.e., a fixed activation) $\varphi$ with appropriately chosen shifts $\varepsilon, \delta$ and scaling factors $\lambda_p$, so that
\begin{equation}
\label{eq:kst_sprecher}
    f(x_1,\dots,x_n) = \sum_{q=1}^{2n+1} \varphi\left( \sum_{p=1}^n \lambda_p \phi(x_p + \varepsilon q) + \delta q \right).
\end{equation}
Due to their similarities, KANs tend to be comparable to MLPs for learning tasks, with the same asymptotic complexity and convergence rates \citep{gao2024convergence} and performance in  variety of head-to-head comparisons \citep{zeng2024kan, shukla2024comprehensive, yu2024kan}.


\subsection{Contributions}

Despite the similarities, KANs provide additional structure that can be exploited for efficient training, which MLPs lack: the explicit dependence on an analytical basis allows for the control of approximation properties at each layer. For this work, KANs with a spline basis include spatial locality of activation functions and features at each layer. 

Therefore, we seek to exploit features of KANs that will enable the efficient training of multichannel MLPs.
To do so, this paper makes the following contributions:
\begin{enumerate}
    \item We exploit the relationship between B-splines and ReLU activations to reformulate KANs in the language of multichannel MLPs. We show that KAN architectures are equivalent to multichannel MLPs, where the biases are fixed to match the spline knots.
    \textbf{(Section \ref{sec:formulation})}
    \item We show that KANs expressed with the B-spline basis are a preconditioned version of multichannel MLPs. This preconditioning improves the condition number of the Hessian of the resulting linear system, providing an explanation for the rule-of-thumb that ``feature localization is better for training'' seen particularly in many image processing machine learning tasks \citep{lowe2004distinctive, Xu_2024} , and we comment on the implications of this result in relation to the spectral bias of MLPs. \textbf{(Section \ref{sec:preconditioning})}
    \item We build hierarchical methods for the refinement of B-spline KANs, accelerating the grid transfer methods posed in \citep{liu2024kan}. This method is fast enough to be done during training, with significant improvements in training accuracies as a result. \textbf{(Section \ref{sec:multigrid})}
    \item We construct a method to parameterize free-knot spline KANs. This avoids the elastic data-driven grid deformation strategies employed elsewhere and solves some of the questions raised regarding the dependence of a learned KAN solution on the underlying spline grid \citep{howard2024multifidelity, somvanshi2024survey, abueidda2024deepokan}. \textbf{(Section \ref{sec:free-knot})}
\end{enumerate}

\section{Multichannel MLPs and KANs}
\label{sec:formulation}
In the original KAN paper and many subsequent works, the KAN activations are built from a B-spline basis. We start with the same formulation, although we drop the conventional MLP tail added to most KAN layer implementations. 
Other bases such as wavelets and Chebyshev polynomials have since been proposed; see Appendix \ref{sec:other-bases} for further discussion.

In what follows, let $T = \{t_i\}_{i=1-r:n+r-1}$ be a strictly ordered set of $n+2r-1$ spline knots, where $t_i < t_{i+1}$, with $t_0 = a$ and $t_n = b$.
Let $\mathbb{P}^k$ denote the space of polynomials of degree $k$.
We recall some terminology and results regarding the theory of splines, c.f. \citep{chui1988multivariate}.
\begin{definition}
The \emph{spline space of order $r$ with knots $T$}, denoted by $S_r(T)$, is the set of functions
\begin{equation}
S_r(T) = \{ f \in C^{r-2}([a,b]) \,\,:\,\, f\lvert_{[t_i, t_{i+1}]} \in \mathbb{P}^{r-1} \}.
\end{equation}
\end{definition}
\begin{remark}
The spline space of order $r$ consists of piecewise polynomials of degree $r-1$ (subject to additional smoothness constraints).
\end{remark}

\begin{lemma}[\citet{chui1988multivariate}]
Define the functions
$b^{[1]}_i(x) = \begin{cases} 1 & x \in [t_i, t_{i+1}] \\ 0 & \text{else} \end{cases}.$
Following \citep{deboor1978practical}, recursively define
\begin{equation}\label{eq:br-def}
    b^{[r]}_i(x) = \frac{x - t_i}{t_{i+r-1} - t_i} b^{[r-1]}_i(x) + \frac{t_{i+r} - x}{t_{i+r} - t_{i+1}} b^{[r-1]}_{i+1}(x).
\end{equation}
Then,
    $B_S = \{ b^{[r-1]}_i(x) \}_{i=1-r:n-1}$
is a basis for $S_r(T)$.
\end{lemma}
\begin{lemma}[\citet{chui1988multivariate}]
 \label{lemma:basis}
The set
        $B_R = \{ \relu( \cdot - t_i)^{r-1} \}_{i=1-r}^{n-1}$
is an equivalent basis for $S_r(T)$.
\end{lemma}
As a result of this basis equivalence, we can express our KAN layers equivalently in either $B_S$ or $B_R$. For $B_S$, i.e. the spline basis, there is a weight 3-tensor $\widetilde{W}^{(\ell)}$ such that
\begin{equation} \label{eq:kan-spline-basis}
x_q^{(\ell+1)} = \sum_{p=1}^P \sum_{i=1-r}^{n-1} \widetilde{W}^{(\ell)}_{qpi} \, b_i^{[r-1]}(x_p^{(\ell)}).
\end{equation}
Equivalently, since each activation $\phi_{pq}^{(\ell)} \in S_r(T)$, there exists a weight 3-tensor $W^{(\ell)}$ so that
\begin{equation}
    \label{eq:kan_relu}
   x_q^{(\ell+1)} = \sum_{p=1}^P \sum_{i=1-r}^{n-1} W^{(\ell)}_{qpi} \, \relu(x^{(\ell)}_p - t_i)^{r-1} . 
\end{equation}
Moreover, since the bases $B_S$ and $B_R$ are equivalent, there exists an invertible linear mapping $A \in \mathbb{R}^{(r+n-1) \times (r+n-1)}$ between the two bases, such that
\begin{equation}
W^{(\ell)} = \widetilde{W}^{(\ell)} \times_3 A, 
\end{equation}
where $\times_j$ denotes the $j$-mode tensor product.
The change of basis matrix for arbitrary order $r$ and uniform or nonuniform knots is derived explicitly in Appendix \ref{app:ntk_proof}.

The case $r=2$ explicitly yields the standard ReLU activation. The connection between ReLU MLPs and continuous piecewise linear functions has been made previously \citep{arora2018understanding, he2020relu, opschoor2020deep, devore2021neural}. We present this connection again, in light of the equivalent basis result for splines and ReLU networks for use specifically in KAN architectures. 
In what follows, we apply the linear weights at each layer \textit{after} the nonlinear activation, instead of inside. This notational shift is for convenience of exposition, and does not result in a change of the architecture. For a detailed analysis of how this notation yields the same MLP, see Appendix \ref{app:notational_shift}.

An MLP layer $\ell$ with $P$ input neurons and $Q$ output neurons with a $\relu$ activation can be written as
\begin{equation} \label{eq:MLP-equivalent}
    \begin{split}
    x_q^{(\ell+1)} &= \sum_{p=1}^P W_{qp}^{(\ell)} \relu(x_p^{(\ell)} - t_p) 
    = \sum_{p=1}^P \sum_{i=1}^P \delta_{i=p} W_{qpi}^{(\ell)} \relu( x_p^{(\ell)} - t_i ),
    \end{split}
\end{equation}
where $\delta$ denotes the Dirac delta function.
Comparing this to Equation \eqref{eq:kan_relu}, we find two informative perspectives. First, we see that MLPs follow a similar architecture to KANs but use a weight tensor that is restricted into the low-rank factorization $W^{(\ell)}_{qpi} = \delta_{i=p} W^{(\ell)}_{qp}$. This factorization limits the expressiveness of standard MLPs compared to KANs with the same number of hidden nodes (although the KAN architecture will have more parameters). 


Second, we see that a MLP layer and a KAN layer differ in that a KAN layer has multiple channels, each featuring a different bias to perform a translation as defined by the spline grid knots. The channel dimension is labeled by $i$  in Eqs.~\eqref{eq:MLP-equivalent} and~\eqref{eq:kan-spline-basis} defining the layer's action. For a KAN the effect is to provide greater localization of features in space based on the compact support of the spline basis. In principle a classic ReLU MLP with a similar parameter count includes in its approximation space the KAN. However, the global nature of the ReLU basis implies that perturbations to coefficients weight and bias parameters have a global impact on the network output. This impedes both interpretability of the weights, and relates to the spectral bias observed in training. 

Each architecture induces separating hyperplanes in the state space at layer $\ell$ by nature of where ReLU activation functions change from zero to positive. This viewpoint is illustrated in Figure \ref{fig:spline-support}, looking at the first layer of a KAN with inputs in $\mathbb{R}^2$, and comparing to an MLP in the form of Equation \eqref{eq:MLP-equivalent}. For standard MLPs that take the conventional ordering of affine layer followed by the nonlinear activation, there are $Q$ separating hyperplanes (i.e. one from each neuron) and $\tiny{\sum_{i=0}^P \begin{pmatrix} Q \\ i \end{pmatrix}}$ distinct subregions. In the reordered MLP case, we have a single bias per dimension of the input, which results in $P$ orthogonal separating hyperplanes that subdivide space into $2^P$ convex subregions. In contrast, for KANs, each axis is partitioned on a grid, with an additional hyperplane at each spline knot, resulting in $(n-1)P$ hyperplanes that subdivide space into $n^P$ convex subregions; each parallel hyperplane arises from the extra channels that encode the shifts from each spline knot.

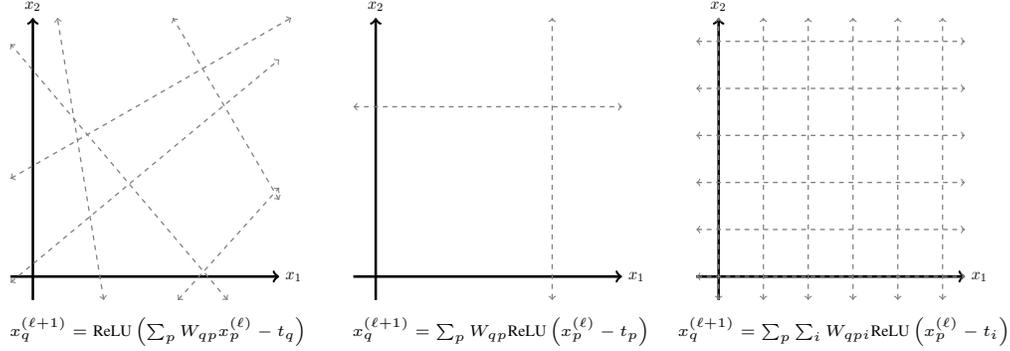
\begin{figure*}[tbp!]
\centering
\begin{subfigure}[b]{0.32\textwidth}
\centering
\resizebox{1.6in}{1.6in}{
\begin{tikzpicture}
\pgfmathsetseed{123} 
\def\rows{5}
\def\cols{5}

\draw[->, ultra thick] (-0.5,0) -- (\cols + 0.5,0) node[right] {$x_1$};
\draw[->, ultra thick] (0,-0.5) -- (0,\rows + 0.5) node[above] {$x_2$};

\pgfmathsetmacro{\xpos}{rand*5} 
\pgfmathsetmacro{\ypos}{rand*5} 
\draw[dashed, gray, thick, <->] (0.5*\xpos, -0.5) -- (\ypos, 5.5); 

\pgfmathsetmacro{\xpos}{rand*5} 
\pgfmathsetmacro{\ypos}{rand*5} 
\draw[dashed, gray, thick, <->] (-0.5,\xpos) -- (5.5,-\ypos); 

\pgfmathsetmacro{\xpos}{rand*5} 
\pgfmathsetmacro{\ypos}{rand*5} 
\draw[dashed, gray, thick, <->] (\xpos,-0.5) -- (5.5,\ypos); 

\pgfmathsetmacro{\xpos}{rand*5} 
\pgfmathsetmacro{\ypos}{rand*5} 
\draw[dashed, gray, thick, <->] (-2.5*\xpos,-0.5) -- (-0.5,\ypos); 

\pgfmathsetmacro{\xpos}{rand*5} 
\pgfmathsetmacro{\ypos}{rand*5} 
\draw[dashed, gray, thick, <->] (-0.5,-\xpos) -- (1.5*\ypos,5.5); 

\pgfmathsetmacro{\xpos}{rand*5} 
\pgfmathsetmacro{\ypos}{rand*5} 
\draw[dashed, gray, thick, <->] (5.5,\xpos) -- (2.0*\ypos,5.5); 

\end{tikzpicture}
}
\caption*{\tiny{$x^{(\ell+1)}_q = \text{ReLU}\left(\sum_p W_{qp} x^{(\ell)}_p - t_q\right)$}}
\end{subfigure}
\begin{subfigure}[b]{0.32\textwidth}
\centering
\resizebox{1.6in}{1.6in}{
\begin{tikzpicture}

\pgfmathsetseed{126} 
\def\rows{5}
\def\cols{5}

\draw[->, ultra thick] (-0.5,0) -- (\cols + 0.5,0) node[right] {$x_1$};
\draw[->, ultra thick] (0,-0.5) -- (0,\rows + 0.5) node[above] {$x_2$};

\pgfmathsetmacro{\xpos}{rand*5}
\pgfmathsetmacro{\ypos}{rand*5}
\draw[dashed, gray, thick, <->] (-0.5, \xpos) -- (5.5, \xpos);
\draw[dashed, gray, thick, <->] (\ypos, -0.5) -- (\ypos, 5.5);

\end{tikzpicture}
}
\caption*{\tiny{$x_q^{(\ell+1)} = \sum_p W_{qp} \text{ReLU}\left(x^{(\ell)}_p - t_p\right)$}}
\end{subfigure}
\begin{subfigure}[b]{0.32\textwidth}
\centering
\resizebox{1.6in}{1.6in}{
\begin{tikzpicture}

\pgfmathsetseed{123} 
\def\rows{5}
\def\cols{5}

\draw[->, ultra thick] (-0.5,0) -- (\cols + 0.5,0) node[right] {$x_1$};
\draw[->, ultra thick] (0,-0.5) -- (0,\rows + 0.5) node[above] {$x_2$};

\foreach \i in {0,1,...,\rows} {
    \draw[dashed, gray, thick, <->] (-0.5,\i) -- (\cols+0.5,\i); 
}
\foreach \j in {0,1,...,\cols} {
    \draw[dashed, gray, thick, <->] (\j,-0.5) -- (\j,\rows+0.5); 
}
\end{tikzpicture}
}
\caption*{\tiny{$x_q^{(\ell+1)} = \sum_p \sum_i W_{qpi} \text{ReLU}\left(x^{(\ell)}_p - t_i\right)$}}
\end{subfigure}
\caption{Example of the learned separating hyperplanes  for three different architecture choices: (left) a standard MLP layer, where the nonlinearity is applied after the affine transformation, (center) a reordered MLP layer, where the nonlinearity is applied before the linear transformation, and (right) a standard KAN layer, where the nonlinearity is applied before the linear transformation.\label{fig:spline-support}}
\end{figure*}


\section{Preconditioning via Change of Basis}
\label{sec:preconditioning}
The change of basis that occurs in Equation \eqref{eq:kan_relu} can be viewed as preconditioning the optimization problem encountered during training. 
Let $\mathcal{W}$ denote the set of all weights for the network with the ReLU basis, and let $\widetilde{\mathcal{W}}$ be the transformed weights for the spline basis. Denote by $\mathcal{A}$ the transformation from $\widetilde{\mathcal{W}}$ to $\mathcal{W}$.
In detail, for $\widetilde{\mathcal{W}} = \left\{ \widetilde{W}^{(\ell)} \right\}_{\ell = 0,\dots,L}$ define
\begin{equation}
\mathcal{A}\left(\widetilde{\mathcal{W}}\right) \coloneqq \mathcal{W} = \left\{W^{(\ell)} \right\}_{\ell = 0,\dots,L} 
= \left\{\widetilde{W}^{(\ell)} \times_3 A^{(\ell)} \right\}_{\ell = 0,\dots,L},
\end{equation}
where $A^{(\ell)}$ is the appropriate change of basis on layer $\ell$ (depending on $P,Q$, and spline order).
For an objective function $\mathcal{L}$, training implies the equivalent  minimization problems 
\begin{equation} \label{eq:min_COB}
\min_{\mathcal{W}} \mathcal{L}\left(\mathcal{W}\right) = \min_{\widetilde{\mathcal{W}} } \left( \mathcal{L}\circ \mathcal{A} \right) \left(\widetilde{\mathcal{W}}\right) .
\end{equation}
When we perform gradient descent iterations in $\mathcal{W}$ with a step size $\eta$, for the weights of each layer, we take steps of the form
$    W^{(\ell)}_{k+1} = W^{(\ell)}_k - \eta \nabla \mathcal{L}\left(W^{(\ell)}_k\right),$
while if we perform gradient descent iterations in $\widetilde{\mathcal{W}}$:
\begin{equation} \label{eq:tensor_precondition} \begin{split}
    \widetilde{W}^{(\ell)}_{k+1} &= \widetilde{W}^{(\ell)}_k - \eta \nabla (\mathcal{L} \circ \mathcal{A})\left(\widetilde{W}^{(\ell)}_k\right) 
    = \widetilde{W}^{(\ell)}_k - \eta  \nabla \mathcal{L} \left( \widetilde{W}^{(\ell)}_k \times_3 A \right) \times_3 A^T \\
    \Longleftrightarrow \qquad     W^{(\ell)}_{k+1} &= W^{(\ell)}_k - \eta  \nabla \mathcal{L} \left(W^{(\ell)}_k\right) \times_3 A^TA.
\end{split} \end{equation}
The implication is that performing gradient descent for the spline basis, is equivalent to a preconditioned gradient descent in the ReLU basis. 

Equivalently, we can write Equation \eqref{eq:tensor_precondition} in a vectorized fashion. To do so, vectorize in mode order (following \citep{kolda2009tensor}) the weight tensors $\mathcal{W}$, or $\widetilde{\mathcal{W}}$, and then concatenate the vector of weights from each layer; denote the resulting vector of weights $\mathbb{W}$, or $\widetilde{\mathbb{W}}$, respectively. Similarly, create a square block-diagonal matrix $\mathbb{A}$ of  $\sum_{\ell=0}^L P_\ell Q_\ell (n_\ell + r_\ell -1)$ rows and columns, whose diagonal consists of $\sum_\ell P_\ell Q_\ell $ blocks, each of whom are the corresponding linear transform $A$, i.e.
\begin{equation}\label{eq:bbW}
\mathbb{W} = \mathbb{A} \widetilde{\mathbb{W}}.
\end{equation}
With this notation, the final line in Equation \eqref{eq:tensor_precondition} becomes the preconditioned gradient descent update
\begin{equation} \label{eq:vector_precondition} \begin{split}
\mathbb{W}_{k+1} &= \mathbb{W}_k - \eta  \mathbb{A} \mathbb{A}^T \nabla \mathcal{L}\left( \mathbb{W}_k \right).
\end{split} \end{equation}


We illustrate the effects of this preconditioning for a least-squares regression problem posed in the spline vs. the $\text{ReLU}^{r-1}$ basis. With data $\{(X_d, Y_d)\}_{d=1:D}$, where $X_d$ are randomly sampled from a domain $\Omega$ with uniform measure $\mu$, the least squares problem is
\begin{equation} \label{eq:lstsq}
\min_{W} \frac{1}{2D} \left \lVert W \Phi_*(X) - Y \right \rVert_2^2
\end{equation}
where $* \in \{S,R\}$ so that $\Phi_{R_i}(x_j) = b_i(x_j)$ for the spline basis and $\Phi_{S_i}(x_j) = \relu( x_j - t_i )^{r-1}$ for the $\text{ReLU}^{r-1}$ basis. 
Gradient descent converges with a rate of 
$O(\kappa(H_*))$, where $H_*$ is the Hessian of the objective function in Equation \eqref{eq:lstsq} and $\kappa$ denotes its condition number:
\begin{equation} \begin{split}
    H_R &= \frac{1}{D} \Phi_R(X) \Phi_R(X)^T, \\
    H_S &= \frac{1}{D} \Phi_S(X) \Phi_S(X)^T = A H_R A^T.
\end{split} \end{equation}
As the amount of randomly sampled data $D \rightarrow \infty$,  the Hessians converge to the Gram matrices
\begin{equation}
    G_* = \int_\Omega \Phi_*(x) \Phi_*(x)^T d \mu(x).
\end{equation}

\begin{proposition} Suppose the ratio between the maximal and minimal distance between spline knots is bounded above and below, away from $0$.
The condition number of the Gram mass matrices 
are bounded with respect to the number of functions in our basis for the spline case $* = S$, while it is unbounded for the $\text{ReLU}^{r-1}$ case $* = R$.
\end{proposition}
\begin{proof}
For $r=2$, i.e. continuous piecewise linear functions, the boundedness of the condition number $\kappa(G_S)$ is proven for uniform splines as in \citep{hong2022activation}, and for non-uniform non-degenerate splines in \citep{unser1999splines, deboor1978practical}, while the unboundedness of $\kappa(G_R)$ is proven in the appendices of \citep{hong2022activation} for uniform biases within the ReLU activations, and for arbitrary biases in several works, including \citep{goel2017reliably}.
For the case $r > 2$, the result that $\kappa(G_S)$ is bounded follows from a straightforward modification of the argument presented in \citep{hong2022activation}, replacing the diagonal matrices for hat functions with those that define the B-spline basis. For nonuniform splines, the rest of the argument relies on a comparable application of the Gershgorin circle theorem to bound the eigenvalues both above and below, ensuring the condition number is bounded independent of the number of splines, and then using Ostrowski's Theorem \citep{ostrowski1959quantitative} to complete the proof via the same argument as presented in \citet{hong2022activation}.
\end{proof}

Beyond a least-squares problem, in the presence of depth, the transformation from the spline basis to/from the $\text{ReLU}^{r-1}$ basis is done independently for each layer $W^{(\ell)}$. This is the result in Equation \eqref{eq:min_COB}, where the resulting preconditioner is applied block-wise independently for each layer, thereby improving the conditioning and convergence of training for each weight tensor, one layer at a time.


By preconditioning our optimization problem, we expect a change in spectral bias as well. Towards this end, the work in \citep{hong2022activation} provides a rigorous result on shallow MLPs with fixed biases and no linear transformation of weights for the initial layer. The KAN architecture as originally conceived in \citep{liu2024kan} matches the analysis in \citep{hong2022activation}, where the shifts in biases are incorporated through multiple channels instead of as a fixed MLP bias, and in this lens, we comment on the effects of preconditioning on the spectrum of the NTK \citep{rahaman2019spectral}.

When we use the spline basis vs. the ReLU basis, we first notice that the dominant eigenvalue is the same order for both bases:
\begin{theorem} \label{thm:ntk_cob} For symmetric positive semidefinite $M$ let $\rho(M)\geq 0$ denote the spectral radius. Define the MLP and (globally) order-$r$ spline NTK matrices, respectively as
\begin{equation} 
    \text{NTK}_R = J J^T, \hspace{4ex}
    \text{NTK}_S^{[r]} = J \mathbb{A} \mathbb{A}^T J^T,
\end{equation}
where $\mathbb{A}$ \eqref{eq:bbW} depends on $r\in\mathbb{Z}^+$. Then for uniform knot spacing,  
\begin{equation}
    \rho(\text{NTK}_S^{[r]}) \leq 4 \rho(\text{NTK}_R).
\end{equation}
\end{theorem}
\begin{proof}
See Appendix \ref{app:ntk_proof}.
\end{proof}

We secondly notice that the preconditioning evident in Equation \eqref{eq:tensor_precondition} is effectively a structured linear transform that couples weights associated with adjacent knots. This coupling allows more complex paths through the gradient landscape and provides locality -- and therefore dampening -- to the rest of the spectral modes during training, reversing the behavior of spectral bias. Rigorously proving this statement is difficult due to the nullspace of the Hessian and NTK that occurs during training; see Section \ref{app:nullspace} in the Appendix for a concise description of why such a nullspace arises. However, we see this behavior in our experiments, manifesting in an accelerated convergence rate, as seen during our examples presented in Section \ref{sec:experiments}.

For an analysis of the computational cost of the new basis, please see Appendix \ref{app:computational-cost}.

\section{Geometric Refinement}
\label{sec:multigrid}

Using multiresolution features of splines to build KANs at differing resolutions was introduced in \citep{liu2024kan}. However, the refinement done in \citep{liu2024kan} does not perform geometric refinement; instead, they interpolate between grids of arbitrary sizes. To do so, they solve for each hidden node a least-squares problem of size $N_\text{data} \times (n+r-1)$, which is unsurprisingly costly. This idea was dismissed by  the authors as being too slow for practical applications.
In contrast, it is substantially more efficient to refine geometrically and then use restriction and prolongation operators from multigrid literature \citep{hollig2003finite, hackbusch2013multi} to perform the grid transfer, thus enabling multiresolution training to become much more feasible and cost-effective.

\subsection{Multilevel Training of KANs}
We first make the observation that for two sets of knots $T,\,T'$ where $T \subset T'$, we have $S_r(T) \subset S_r(T')$. Thus there is an easily constructed injection operator from one grid to another (in the case of nested grids), denoted $\mathcal{I}$.

The splines defined by an arbitrary set of knots $T$ on arbitrary coarse grid with basis $\{b_{i,T}^{r-1}\}_{i=1-r}^{n-1}$
has a corresponding basis with knots $T'$ on a finer grid given by $\{ b_{i,T'}^{r-1} \}_{i=1-r}^{m-1}$ with $m>n$.
For a KAN layer with knots $T$ and weights $\widetilde{W}_{qpi,T}$,
we define a fine KAN layer with knots $T'$ and weights $\widetilde{W}_{qpi,T'}$ by
\begin{equation}
x_q^{(\ell+1)} = \sum_{p=1}^{P} \sum_{i=1-r}^{m-1} \widetilde{W}_{qpi,T'}^{(\ell)} \, b_{i,T'}^{r-1}(x_p^{(\ell)}),
\label{eq:kan-refinement}
\end{equation}
where for $p=1,\dots,P$, $q=1,\dots,Q$
\begin{equation}
    \widetilde{W}_{qpi,T'}^{(\ell)} b_{i,T'}^{r-1} =
    \sum_{i=1-r}^{n-1} \mathcal{I}_{ij} \widetilde{W}_{qpi,T}^{(\ell)} b_{i,T}^{r-1}
\label{eq:kan-interpolation}
\end{equation}
is an exact equality with interpolation weights $\mathcal{I}_{ij}$.

The refinement process described in
Equations~\eqref{eq:kan-refinement},\eqref{eq:kan-interpolation}
produces a new KAN layer that is equal to the previous
with more trainable parameters.
This is the most important feature of multilevel KANs,
as it means the refinement process does not modify
the training progress made on coarse models. 

\subsection{The Class of Multilevel KANs}
However, the multilevel method for KANs is not unique to our B-spline approach.
This approach is valid for broad classes of KANs where activation functions are discretized.
Any discretization of activation functions with a similar nesting of spaces
upon grid (or polynomial) refinement satisfies this same result, which we will state below.

Let $\mathcal{B} = \{\phi_i:[a,b]\to\mathbb{R}\}_{i=1}^{n_B}$ be a basis of size $n_B$
for the activation function space (e.g. B-splines, Chebyshev, wavelets)
on a domain $[a,b]$ with $a,b\in\mathbb{R}$, $a<b$,
and let $\mathcal{B}' = \{\phi_i':[a,b]\to\mathbb{R}\}_{i=1}^{n_{B'}}$
be a basis of size $n_{B'}>n_B$ such that $\mathcal{B}' \supset \mathcal{B}$.
Then there is an operator $\mathcal{I_B^{B'}}:\mathcal{B}\to \mathcal{B}'$ which interpolates
basis $\mathcal{B}$ to $\mathcal{B}'$ such that for $\sigma\in \mathcal{B}$
we have $\mathcal{I}_B^{B'}\sigma\in \mathcal{B}'$
and $\sigma(x) \equiv \mathcal{I}_B^{B'}\sigma(x)$.

\begin{lemma} \label{lemma:multilevelkans}
Any KAN with activations defined over a basis which is nested under a refinement procedure
may be refined so that the activation functions after refinement are identical to their
pre-refinement counterparts by interpolating the network parameters according to the above procedure.
\end{lemma}

\begin{remark}
This network refinement is tied to the discretization of the activation functions; essentially this performs refinement with respect to the channel dimension of the weight tensors, instead of refining the number of hidden nodes at each layer.
While there are efforts to refine networks by increasing the number of layers or hidden nodes,
the method in Lemma \ref{lemma:multilevelkans} does not change the number of nodes in the network.
\end{remark}



\section{Free-Knot Spline KANs}
\label{sec:free-knot}


The grid dependence of KANs can be relaxed through the use of free-knot splines, instead of relying on a fixed spline grid. We present here a scheme for building free-knot splines, adding an additional trainable parameter to determine the spline knots' locations for each KAN layer. These trainable spline knots - if parameterized directly as a trainable variable - must be restricted (or sorted) to avoid the inversion of spline knots. Instead, following the strategy in \citep{actor2024data}, we parameterize the relative distance \emph{between} spline knots, and then remap them to properly scale to the entire domain of the layer. 

With the domain $[a,b]$ of our splines fixed, we seek to parameterize the set of knots $T$. For the knots $\{t_0,\dots,t_n\}$, we define a trainable parameter $s \in \mathbb{R}^{n}$, and then simply define
\begin{equation} 
t_i = a + (b-a) \sum_{j=1}^i \texttt{softmax}(s)_j.
\label{eq:trainable-knots}
\end{equation}
With this parameterization, it is clear that $t_0 = a$ (since no summands are added), $t_{n} = b$ (since the sum of a softmax is 1), and that for all $i$, $t_i < t_j$ (since softmax is positive), regardless of how the parameter $s$ is updated during training. We initialize $s$ to zero, so that the initial knots are uniformly spaced, and then training can adapt the locations of $T$ to best match the data. For the knots less than $t_0$, we repeat this construction, fixing $t_{1-r} = t_0 - (b-a)$; for the knots greater than $t_n$ we repeat again, fixing $t_{n+r-1} = t_n + (b-a)$.

As the equivalence results between KANs and multichannel MLPs presented earlier do not assume that the splines are uniformly spaced, we can then replace the evaluation of the bases $\Phi_R$ or $\Phi_S$ with these trainable knots, with no other changes to an implementation of a KAN layer. The geometric refinement results do not assume uniform spline knot spacing, and thus can be used in conjunction with free-knot splines without additional complications.

\section{Experiments}
\label{sec:experiments}
We demonstrate results for two regression problems and one scientific machine learning problem, a physics-informed neural network (PINN). Each is presented below; details regarding problem setup, data generation, and hyperparameters can be found in Appendix \ref{app:hyperparameters}.

\begin{table}[tbp!]
\caption{
Accuracy for nonsmooth regression problem, comparing the effects of (1) preconditioning the basis of a multichannel MLP from a ReLU to a spline basis; (2) geometric of the spline knots for the KAN basis, and (3) employing trainable free-knot splines for the KAN basis. MLPs are shown for comparison; standard deviations are computed over random initializations ($N=5$).}
\label{table:results-regression}
\vskip 0.15in
\begin{center}
\begin{small}
\begin{sc}
\begin{tabular}{lcllccc}
\toprule
\multirow[b]{2}{*}{Type} & \multirow[b]{2}{*}{Layers} & \multirow[b]{2}{*}{Basis} & \multirow[b]{2}{*}{Fidelity} & Free & \# & MSE \\
 &  &  &   &   Knots? &  Param. &  mean (stdev)  \\
\midrule
KAN & $[2,5,1]$ & ReLU & coarse     & &  55 & $1.10 \times 10^{-2} \, (8.32 \times 10^{-3})$\\
KAN & $[2,5,1]$ & ReLU & fine       & & 230 & $1.35 \times 10^{-0} \, (1.50 \times 10^{-1})$\\
KAN & $[2,5,1]$ & ReLU & multilevel & & 230 & $1.06 \times 10^{-2} \, (8.03 \times 10^{-3})$\\
\midrule
KAN & $[2,5,1]$ & Spline & coarse     & &  55 & $1.65 \times 10^{-3} \, (4.18 \times 10^{-5})$\\
KAN & $[2,5,1]$ & Spline & fine       & & 230 & $2.54 \times 10^{-3} \, (5.64 \times 10^{-3})$\\
KAN & $[2,5,1]$ & Spline & multilevel & & 230 & $\mathbf{ 3.67 \times 10^{-5}} \, (7.19 \times 10^{-5})  $\\
\midrule
KAN & $[2,5,1]$ & ReLU   & coarse & \checkmark & 110 & $1.75 \times 10^{-3} \, (1.23 \times 10^{-3})$\\ 
KAN & $[2,5,1]$ & ReLU   & fine   & \checkmark & 460 & $9.71 \times 10^{-1} \, (1.74 \times 10^{-1})$\\
\midrule
KAN & $[2,5,1]$ & Spline & coarse & \checkmark & 110 & $1.84 \times 10^{-4} \, (2.18 \times 10^{-4})$\\
KAN & $[2,5,1]$ & Spline & fine   & \checkmark & 460 & $\mathbf{ 2.52 \times 10^{-5}} \, (3.18 \times 10^{-5}) $\\
\midrule
MLP & $[2,5,1]$     & ReLU & & &  20 &   $2.94\times 10^{-2} \, (1.60 \times 10^{-2})$\\
MLP & $[2,30,1]$    & ReLU & & & 120 &   $1.02 \times 10^{-3} \, (9.42 \times 10^{-4})$\\
MLP & $[2,20,20,1]$ & ReLU & & & 500 &   $3.33 \times 10^{-4} \, (2.92 \times 10^{-4})$ \\
\bottomrule
\end{tabular}
\end{sc}
\end{small}
\end{center}
\vskip -0.1in
\end{table}

\begin{table}[tb]
\caption{Accuracy for XOR regression problem, comparing the effects of (1) the choice of spline vs. ReLU basis; (2) geometric refinement of the spline knots for the KAN basis, and (3) use of free-knot splines for the KAN basis. Standard deviations computed with random initializations ($N=5$).}
\label{table:results-regression-xor-deep}
\vskip 0.15in
\begin{center}
\begin{small}
\begin{sc}
\begin{tabular}{lcllccc}
\toprule
\multirow[b]{2}{*}{Type} & \multirow[b]{2}{*}{Layers} & \multirow[b]{2}{*}{Basis} & \multirow[b]{2}{*}{Fidelity} & Free & \# & MSE \\
 &  &  &   &   Knots? &  Param. &  mean (stdev)  \\
\midrule
KAN & $[2,5,5,1]$ & ReLU &  coarse &  & 250 &               $2.13\times 10^{-3} (1.09\times 10^{-3})$ \\
KAN & $[2,5,5,1]$ & ReLU &  fine &  & 1300 &               $1.46\times 10^{-0} (3.16\times 10^{-1})$ \\
KAN & $[2,5,5,1]$ & ReLU &  multilevel &  & 1300 &          $3.18\times 10^{-3} (1.12\times 10^{-3})$ \\
\midrule
KAN & $[2,5,5,1]$ & Spline  & coarse &  & 250 &             $1.37\times 10^{-4} (5.31\times 10^{-5})$ \\
KAN & $[2,5,5,1]$ & Spline  & fine &  & 1300 &             $2.78\times 10^{-2} (3.84\times 10^{-2})$ \\
KAN & $[2,5,5,1]$ & Spline  & multilevel &  & 1300 &        $\mathbf{2.79\times 10^{-6}} (1.24\times 10^{-6})$ \\
\midrule
KAN & $[2,5,5,1]$ & ReLU &  coarse &  \checkmark & 360 &    $1.98\times 10^{-2} (3.73\times 10^{-2})$ \\
KAN & $[2,5,5,1]$ & ReLU &  fine &  \checkmark & 1760 &    $1.34\times 10^{-0} (2.44\times 10^{-1})$ \\
\midrule
KAN & $[2,5,5,1]$ & Spline  & coarse &  \checkmark & 360 &  $\mathbf{4.62\times 10^{-6}} (2.11\times 10^{-6})$ \\
KAN & $[2,5,5,1]$ & Spline  &fine &  \checkmark & 1760 &  $4.96\times 10^{-2} (4.05\times 10^{-2})$ \\
\midrule
MLP & $[2,5,5,1]$ & ReLU &  &  & 50 &                         $3.54\times 10^{-2} (2.86\times 10^{-2})$ \\
MLP & $[2,40,40,1]$ & ReLU &  &  & 1800   &                   $3.49\times 10^{-5}(1.69\times 10^{-5})$ \\
\bottomrule
\end{tabular}
\end{sc}
\end{small}
\end{center}
\vskip -0.1in
\end{table}

We experiment with different architectural choices for two regression problems. The first is a $[2,5,1]$ architecture nonsmooth problem, to highlight the necessity of mesh adaptation, regressing a 0.175 radian counterclockwise coordinate rotation of the function
\begin{equation}
f(x,y) = \cos(4\pi x) + \sin(\pi y) + \sin(2\pi y) + |\sin(3\pi y^2)|.
\label{eq:nonsmooth-regression}
\end{equation}
The second uses a $[2,5,5,1]$ architecture on a smoothed XOR problem, regressing the function
\begin{equation}
f(x,y) = \tanh(20x-10) \, \tanh(20x - 40y + 10).
\label{eq:xor-regression}
\end{equation}
This example notable in that support vector machines are unable to capture the XOR function, and as our networks are both small and shallow for this task, this example gains additional relevance.

We train using L-BFGS on a series of models, each with equivalent amounts of work (FLOPs) during training.
We denote the epochs counts in a list such as $[32,16,8,4]$,
meaning 32 epochs are performed on the initial model,
and then 16 epochs are performed after transferring to a refined model
where the original grid was evenly subdivided once,
and so on.
We refer to the \textit{coarse model} as $[128,0,0,0]$, the \textit{fine model} as $[0,0,0,16]$, and the \textit{multilevel model} as $[32,16,8,4]$. 

\begin{figure}[b]
    \centering
    \includegraphics[trim={6.6cm 1.2cm 6.6cm 1.4cm},clip,width=0.8\textwidth]{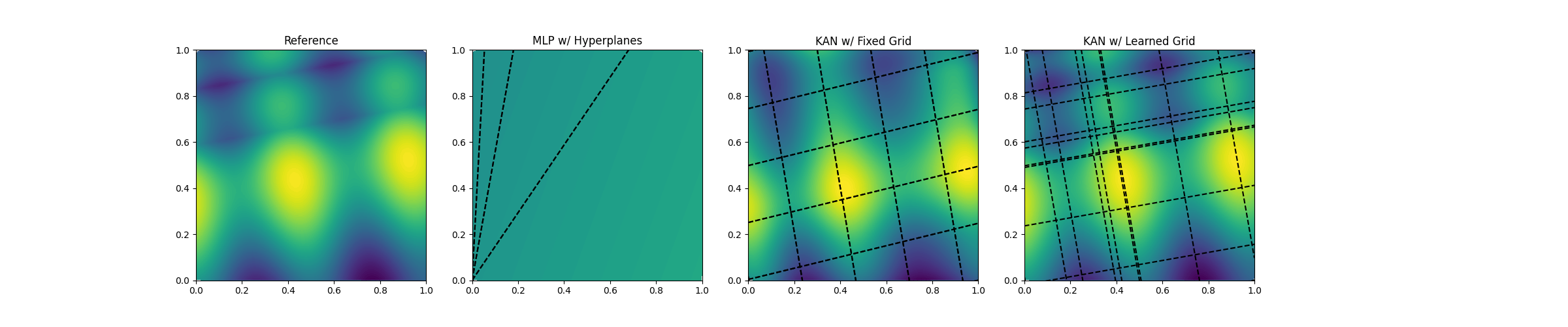}
    \hspace{-1.4cm}
    \includegraphics[trim={11.6cm 1.2cm 0cm 1.4cm},clip,width=0.06\textwidth]{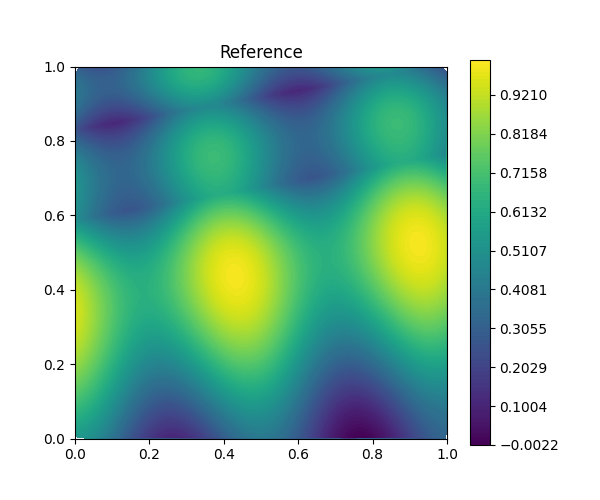} \\
    \includegraphics[trim={6.6cm 1.2cm 6.6cm 1.4cm},clip,width=0.8\textwidth]{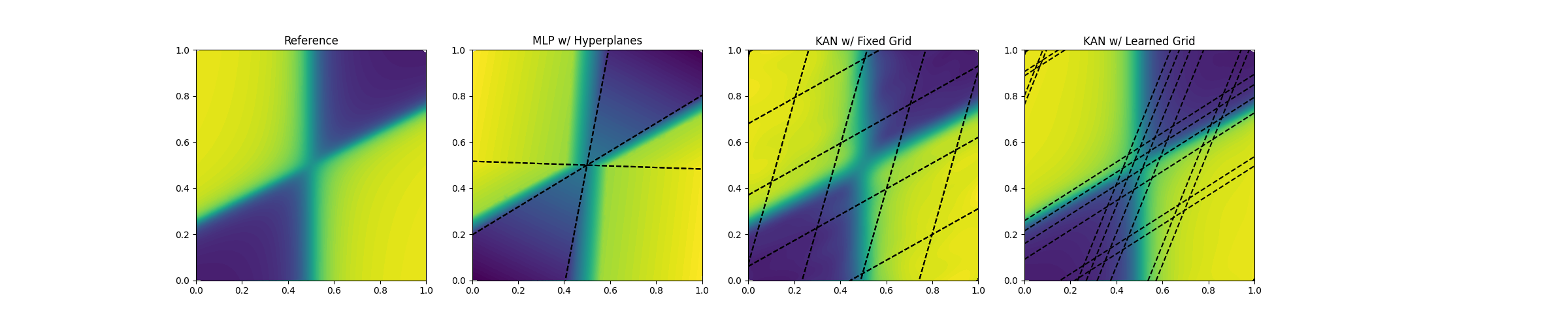}
    \hspace{-1.4cm}
    \includegraphics[trim={11.6cm 1.2cm 0cm 1.4cm},clip,width=0.06\textwidth]{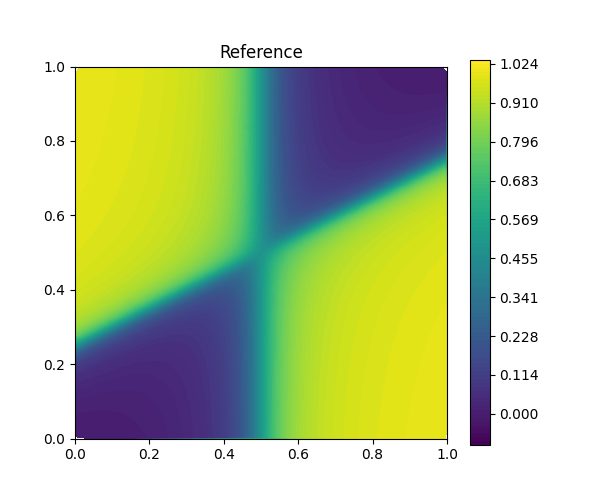}
    \caption{Comparison of ReLU MLP hyperplanes against the KAN spline grid corresponding to two hidden neurons on the nonsmooth function regression example (top) and XOR problem (bottom).}
    \label{fig:learned-grid-comparison-nonsmooth}
\end{figure}

Results are reported in Tables \ref{table:results-regression}-\ref{table:results-regression-xor-deep} and in Figure \ref{fig:convergence-history} in the supplemental materials. For both problems, we generally see that (1) the spline basis consistently outperforms the ReLU basis (and comparable MLPs); (2) we achieve better accuracies for comparable amounts of work when progressing to a multilevel training strategy; and (3) for the nonsmooth problem, where knot adaptation is necessary, the free-knot spline architectures outperform their fixed-knot counterparts by a significant margin.
The learned separating hyperplanes for  are shown in Figure \ref{fig:learned-grid-comparison-nonsmooth} for the nonsmooth problem (top) and the XOR problem (bottom), visualizing the benefits of refinement and knot adaptivity in KAN architectures.

\begin{table}[tbp!]
\caption{Accuracy for PINN problem, comparing the effects of geometric refinement of the spline knots for the KAN basis. Comparable MLPs are included. Standard deviations computed across random initializations ($N=5$).}
\label{table:results-pinn}
\vskip 0.15in
\begin{center}
\begin{small}
\begin{sc}
\begin{tabular}{lclcc}
\toprule
Type & Layers &  Fidelity  & \# Param. & Loss: mean (stdev) \\
\midrule
KAN & $[2,20,20,1]$  & coarse     &  3400 & $6.635 \times 10^{-3} \, (4.060 \times 10^{-3})$\\
KAN & $[2,20,20,1]$  & fine       &  9700 & $4.072 \times 10^{-3} \,(2.514 \times 10^{-3})$\\
KAN & $[2,20,20,1]$  & multilevel &  9700 & $\mathbf{2.402\times 10^{-5}} \,(1.897\times10^{-5})$ \\
\midrule
MLP & $[2,20,20,1]$  & & 500 & $ 2.334 \times 10^{-2} \, (1.740\times 10^{-3})$ \\
MLP & $[2,56,56,1]$  & & 3416 & $ 1.734\times 10^{-2} (7.989\times 10^{-3})$\\
\bottomrule
\end{tabular}
\end{sc}
\end{small}
\end{center}
\vskip -0.1in
\end{table}

We similarly experiment with different architectures for a PINN solving Burger's Equation, with a small perturbation by an elliptic problem for stability. 
We run for 10000 steps for the basis results; for refinement, we use $[3200,0,0]$ epochs for our coarse model, $[0,0,800]$ for the fine model (one level coarser than the regression problems), and $[800,400,200]$ for the multilevel model. We only show results using the spline basis with fixed knots for this problem.

Results are shown in Table \ref{table:results-pinn}. The KAN architectures - even the coarsest ones - outperform the comparable MLP architectures in terms of training loss; as in the regression case, we see the benefits of multilevel training in increased accuracy of several orders of magnitude.

\section{Conclusions}

While KANs are equivalent to multichannel MLPs posed in a spline basis, the spline basis provides a better-posed optimization problem than ReLU MLP equivalents, alongside better accuracy and faster training. While these ideas extend to other MLP-based architectures, we acknowledge that these methods are not applicable to all architectures, and that refinement is only considered in channel dimension, not with regards to depth or width -- these remain beyond the scope of this work.

In the spirit of \citep{shukla2024comprehensive, yu2024kan}, the incorporation of free-knot splines is arguably the ``fairest'' comparison to multichannel MLPs, due to the presence of both trainable weights and biases, and due to the induced preconditioning of the optimization problem, we see better performance of KANs under the same training methods.

\begin{ack}
This work was funded in part by the National Nuclear Security Administration Interlab Laboratory Directed Research and Development program under project number 20250861ER.
This paper describes objective technical results and analysis. Any subjective views or opinions that might be expressed in the paper do not necessarily represent the views of the U.S. Department of Energy or the United States Government. Sandia National Laboratories is a multimission laboratory managed and operated by National Technology and Engineering Solutions of Sandia, LLC., a wholly owned subsidiary of Honeywell International, Inc., for the U.S. Department of Energy's National Nuclear Security Administration under contract DE-NA-0003525. SAND2025-06102O. The research was performed under the auspices of the National Nuclear Security Administration of the U.S. Department of Energy at Los Alamos National Laboratory, managed by Triad National Security, LLC under contract 89233218CNA000001. LA-UR-25-24956.
\end{ack}




\bibliography{references}
\bibliographystyle{icml2025}

\newpage
\appendix
\onecolumn
\section{Notational Shift for MLP Architectures} 
\label{app:notational_shift}

Conventionally, an MLP with $L$ hidden layers, starting with an input $x = x^{(0)}$, has each layer $\ell=0,\dots,L-1$ expressed as
\begin{equation*}
x^{(\ell+1)} = \sigma\left( W^{(\ell)} x^{(\ell)} + b^{(\ell)} \right),
\end{equation*}
with the final output  
\begin{equation*} NN(x) = x^{(L+1)} = W^{(L)} x^{(\ell)}. \end{equation*}
This compositional structure can be broken down more granularly, as
\begin{equation} \begin{split}
x^{(\ell + \frac{1}{2})} &= W^{(\ell)} x^{(\ell)} \\
x^{(\ell + 1)} &= \sigma \left( x^{(\ell + \frac{1}{2})} + b^{(\ell)} \right),\\    
\end{split} \label{eq:granular} \end{equation}
where then the final layer omits the nonlinear activation step so that $ NN(x) = x^{(L + \frac{1}{2})}$ instead of $NN(x) = x^{(L+1)}.$

The rearranged layers used in our analysis, which post-multiply layers by learnable weights instead of pre-multiplying by learnable weights, simply combine the expressions in \eqref{eq:granular} for $\ell = 1, \dots, L-1$ into a single step,
\begin{equation} \begin{split}
x^{(\ell+\frac{1}{2})} &= W^{(\ell)} \sigma\left( x^{(\ell - \frac{1}{2})} + b^{(\ell-1)} \right),
\end{split} \end{equation}
yielding the same computation.

Notationally, we substitute $t^{(\ell)} = - b^{(\ell)}$, to match the conventions from literature on splines. This yields our ultimate expression for each layer, 
\begin{equation}
x^{(\ell+\frac12)} = W^{(\ell)} \sigma\left(x^{(\ell-\frac12)}-t^{(\ell-1)}\right).
\end{equation}


\section{Proof of NTK Spectra Comparison} 
\label{app:ntk_proof}

Consider the extended spline knots $T = \{t_{1-r} < \dots < t_{n+r-1} \}$, to define splines on the interval $[a,b] = [t_0, t_{n+1}]$. For $r=1$, define the functions
\begin{equation}
    b_i^{[1]}(x) = \begin{cases} 1 & x \in [t_i, t_{i+1}] \\
    0 & \text{else}\end{cases},
    \hspace{4ex}
    \psi_i^{[1]}(x) = \begin{cases} 1 & x \ge t_i \\ 0 & \text{else} \end{cases}.
\end{equation}
For $r >1$, we define the spline basis function recursively as before:
\begin{equation}
    b^{[r]}_i(x) = \frac{x - t_i}{t_{i+r-1} - t_i} b^{[r-1]}_i(x) + \frac{t_{i+r} - x}{t_{i+r} - t_{i+1}} b^{[r-1]}_{i+1}(x).
\end{equation}
Each function $b^{[r]}_i$ is supported on the interval $[t_i, t_{i+r}]$.
We define the $\text{ReLU}^{r-1}$ basis functions as
\begin{equation}
    \psi^{[r]}_i(x) = \text{ReLU}(x - t_i)^{r-1}.
\end{equation}

Define $A^{[r]} \in \mathbb{R}^{(n+r-1) \times (n+r-1)}$ as the change-of-basis matrix  between $B^{[r]} = \begin{bmatrix} b^{[r]}_{1-r}, \dots, b^{[r]}_{n-1} \end{bmatrix}^T$ and $\Phi^{[r]} = \begin{bmatrix} \psi^{[r]}_{1-r} , \dots, \psi^{[r]}_{n-1} \end{bmatrix}^T$ so that
\begin{equation}\label{eq:cob}
    B^{[r]} = A^{[r]} \Phi^{[r]}.
\end{equation}
It is easily verified that 
\begin{equation} \label{eq:A_1}
A^{[1]}_{ij} = \begin{cases} 1 & j = i \\ -1 & j = i+1 \end{cases}.
\end{equation}
The following lemma derives a recursive definition of $A^{[r]}$.

\begin{lemma}
    For $r > 1$, let $A^{[r-1]}$ denote splines of order $r-1$ constructed on knots for splines of order $r$. Then the matrix $A^{[r]}$ is defined entry-wise as:
    \begin{equation} \label{eq:A_recurrence}
         A^{[r]}_{ij} = \frac{1}{t_{i+r-1} - t_i} A^{[r-1]}_{ij} - \frac{1}{t_{i+r} - t_{i+1}} A^{[r-1]}_{i+1,j},
    \end{equation}
    where we define $A^{[r-1]}_{n+r,:} = 0$ to account for the special case of the final row $i=n+r-1$. 
\end{lemma}
\begin{proof}
First, we prove the following result regarding ReLU functions.  
\begin{proposition}\label{lemma:relu_power_recursion}
    For any scalars $a,b$, and for $r > 1$,
    \begin{equation}
        (x - a) \text{ReLU}(x - b)^r = \text{ReLU}(x-b)^{r+1} + (b-a) \text{ReLU}(x - b)^r.
    \end{equation}
\end{proposition} 
\begin{proof}
    If $x \le b$, then both sides equal zero. Thus, it suffices to consider $x > b$, where then
    \begin{equation*} \text{ReLU}(x-b) = x-b. \end{equation*}
    Thus, it suffices to show that 
    \begin{equation*} (x-a)(x-b)^r = (x-b)^{r+1} + (b-a)(x-b)^r. \end{equation*}
    Indeed, factoring the righthand side, we see
    \begin{equation} \begin{split}
        (x-b)^{r+1} + (b-a)(x-b)^r &= (x-b)^r \left( (x-b) + (b-a) \right) = (x-b)^r (x-a),
    \end{split} \end{equation}
    which is the desired result.
\end{proof}

We now prove the main statement by induction. For the base case of $r=2$, we have the known (e.g. \citep{hong2022activation}) result that
\begin{equation}
A^{[2]} = \begin{bmatrix}
\frac{1}{t_0 - t_{-1}} & -\left(\frac{1}{t_1 - t_0} + \frac{1}{t_0 - t_{-1}}\right) & \frac{1}{t_1 - t_0} & \\
&\frac{1}{t_1 - t_0} & -\left(\frac{1}{t_2 - t_1} + \frac{1}{t_1 - t_0}\right) & \frac{1}{t_2 - t_1} & \\
&\ddots & \ddots &\ddots & \\
&& \frac{1}{t_{n-1} - t_{n-2}} & -\left(\frac{1}{t_{n} - t_{n-1}} + \frac{1}{t_{n-1} - t_{n-2}}\right) & \frac{1}{t_{n} - t_{n-1}}  \\
&&& \frac{1}{t_{n} - t_{n-1}} & -\left(\frac{1}{t_{n+1} - t_{n}} + \frac{1}{t_{n} - t_{n-1}}\right)   \\
&&&&\frac{1}{t_{n+1} - t_{n}}  \\
\end{bmatrix},
\end{equation}
which can be rewritten entry-wise as
\begin{equation} \label{eq:A_2}
    A^{[2]}_{ij} = \begin{cases} \frac{1}{t_{i+1} - t_i} & j = i \\
    -\left(\frac{1}{t_{i+1} - t_i} + \frac{1}{t_{i+2}-t_{i+1}} \right) & j = i+1 \\
    \frac{1}{t_{i+2}-t_{i+1}} & j = i+2 \\
    0 & \text{else}
    \end{cases}.
\end{equation}
Directly from \eqref{eq:A_1} and \eqref{eq:A_2}, we see that $A^{[2]}$ satisfies Equation \eqref{eq:A_recurrence},
\begin{equation} \begin{split}
    A^{[2]}_{ij} &= \frac{1}{t_{i+1}-t_i} A_{ij}^{[1]} - \frac{1}{t_{i+2}-t_{i+1}} A_{i+1,j}^{[1]}.
\end{split} \end{equation}

We proceed now to the inductive step for $r \ge 2$. Suppose the result holds through $s=1,\dots,r-1$. By the recurrence relation to define the spline functions $b^{[r]}_i$ \eqref{eq:br-def}, we have
\begin{equation} \begin{split}
    b^{[r]}_i(x) &= \frac{x - t_i}{t_{i+r-1} - t_i} b^{[r-1]}_i(x) + \frac{t_{i+r} - x}{t_{i+r} - t_{i+1}} b^{[r-1]}_{i+1}(x) \\
    &= \left( \frac{1}{t_{i+r-1} - t_i} \right) \sum_{j} A^{[r-1]}_{ij} \left( x- t_i \right)  \text{ReLU}(x - t_j)^{r-2} \\
    &\qquad\qquad - \left(\frac{1}{t_{i+r} - t_{i+1}} \right) \sum_{j} A^{[r-1]}_{i+1,j} \left( x - t_{i+r} \right) \text{ReLU}(x - t_j)^{r-2}.
\end{split} \end{equation}
By Lemma \ref{lemma:relu_power_recursion}, 
\begin{equation} \begin{split}
    b^{[r]}_i(x) 
    &= \left( \frac{1}{t_{i+r-1} - t_i} \right) \sum_{j} A^{[r-1]}_{ij} \left( x- t_i \right)  \text{ReLU}(x - t_j)^{r-2} \\
    &\qquad\qquad - \left(\frac{1}{t_{i+r} - t_{i+1}} \right) \sum_{j} A^{[r-1]}_{i+1,j} \left( x - t_{i+r} \right) \text{ReLU}(x - t_j)^{r-2} \\
    &= \left( \frac{1}{t_{i+r-1} - t_i} \right) \sum_{j} A^{[r-1]}_{ij} \left( \text{ReLU}(x - t_j)^{r-1} + (t_j - t_i) \text{ReLU}(x - t_j)^{r-2} \right) \\
    &\qquad\qquad - \left(\frac{1}{t_{i+r} - t_{i+1}} \right) \sum_{j} A^{[r-1]}_{i+1,j} \left( \text{ReLU}(x - t_j)^{r-1} + (t_j - t_{i+r}) \text{ReLU}(x - t_j)^{r-2} \right).
\end{split}
\end{equation}
Since $b_i^r \in C^{r-2}([a,b])$ and $\text{ReLU}(\, \cdot - t_j)^{r-1} \in C^{r-2}([a,b])$ but $\text{ReLU}(\, \cdot - t_j)^{r-2} \notin C^{r-2}([a,b])$,
the coefficients of the $\text{ReLU}(\cdot - t_j)^{r-2}$ terms must cancel. This is specifically because expanding the $\text{ReLU}(\, \cdot - t_j)^{r-2}$ terms as piecewise polynomials, none of the expanded terms are sufficiently differentiable at the knot, and thus all terms must cancel for the full summation to be in $C^{r-2}([a,b])$. This can also be shown directly through laborious arithmetic calculation. Therefore,
\begin{equation} \begin{split}
    b^{[r]}_i(x) 
    &= \left( \frac{1}{t_{i+r-1} - t_i} \right) \sum_{j} A^{[r-1]}_{ij} \text{ReLU}(x - t_j)^{r-1}  - \left(\frac{1}{t_{i+r} - t_{i+1}} \right) \sum_{j} A^{[r-1]}_{i+1,j} \text{ReLU}(x - t_j)^{r-1} \\
    &= \sum_{j} \left( \frac{1}{t_{i+r-1} - t_i}   A^{[r-1]}_{ij}  - \frac{1}{t_{i+r} - t_{i+1}}  A^{[r-1]}_{i+1,j} \right) \text{ReLU}(x - t_j)^{r-1},
\end{split} \end{equation}
which completes the proof.
\end{proof}



\begin{corollary}[Uniform knots]\label{cor:uni}
In the case of uniform knots, with spacing $h = t_i - t_{i-1} = \frac{1}{n}$, $A^{[r]}$ takes the direct form
\begin{equation}\label{eq:uni-Ar}
A^{[r]} = \frac{h^{1-r}}{(r-1)!} (A^{[1]})^r.
\end{equation}
Moreover, $A^{[r]}$ is upper triangular Toeplitz, with 
entries
    \begin{equation}\label{eq:uni-Ar_ij}
        A^{[r]}_{ij} = \begin{cases}
            \frac{(-1)^{i-j} }{(i-j)!\,(r-i+j)!} \, \frac{r}{h^{r-1}}  & i \le j \le i+r \\
            0 & \text{else}
        \end{cases}.
    \end{equation}
\end{corollary}
\begin{proof}
    This follows by noting that if $t_{i+r-1}-t_i = t_{i+r}-t_{i+1} = h(r-1)$, \eqref{eq:A_recurrence} simply corresponds to a left scaling of $A^{[r-1]}$ by $\tfrac{1}{h(r-1)}A^{[1]}$ \eqref{eq:A_1}. Repeating this for $r=2,3,...$ yields \eqref{eq:uni-Ar}. For the Toeplitz property, from \eqref{eq:uni-Ar} we see that $A^{[r]}$ is a power of $A^{[1]}$. Since $A^{[1]}$ is an upper bidiagonal Toeplitz matrix, it follows that $A^{[r]}$ is also upper triangular Toeplitz (see, e.g. \cite[Lemma 11]{southworth2019necessary}), with coefficients given in \eqref{eq:uni-Ar_ij}.
\end{proof}
    
For $r=2$, this is exactly the matrix in \citep{hong2022activation}:
\begin{equation}
A =    \frac1h \begin{bmatrix} 1 & -2 & 1 \\ & 1 & -2 & 1 \\ &&\ddots & \ddots & \ddots \\
&&& 1 & -2 & 1 \\ &&&& 1 & -2 \\ &&&&& 1 \end{bmatrix}.
\end{equation}

\begin{theorem} For symmetric positive semidefinite $M$ let $\rho(M)\geq 0$ denote the spectral radius. 
Then for uniform knot spacing,   
\begin{equation}
    \rho(\text{NTK}_S^{[r]}) \leq 4\rho(\text{NTK}_R).
\end{equation}
\end{theorem}
\begin{proof}
Define an auxiliary equivalent ReLU$^{r-1}$ basis 
\begin{equation}
\varphi_i^{[r]}(x) = \text{ReLU}\left( \frac{x - t_i}{h} \right)^{r-1}.
\end{equation}
As ReLU is a homogeneous function with respect to positive scalars, this basis is equivalent again to the ReLU basis.

We therefore repeat the process of the above construction, yielding an equivalent change-of-basis matrix $\widetilde{A}^{[r]}$, where
\begin{equation}
\widetilde{A}^{[r]} = \frac{1}{h^{1-r}} A^{[r]}.
\end{equation}
Let $J$ denote the Jacobian of $f$ with respect to the weights $W$, so that
\begin{equation} \begin{split}
\text{NTK}_R &= J J^T \\
\text{NTK}_S^{[r]} &= J \mathbb{A}\mathbb{A}^T J^T,
\end{split} \end{equation}
where $\mathbb{A}$ is the reordered block diagonal operator of layer-specific $A^{[r]}$ \eqref{eq:bbW}.  

Note that $JJ^T$ and $J\mathbb{A}\mathbb{A}^TJ^T$ are symmetric with nonnegative eigenvalues, and largest eigenvalue given by the square of the largest singular value of $J$ and $J\mathbb{A}$, respectively. Noting that the largest singular value of an operator is given by the $\ell^2$-norm, by sub-multiplicativity $\|J\mathbb{A}\| \leq \|J\|\|\mathbb{A}\|$. Thus we will bound the spectrum of $JA$ by considering the maximum singular value of the change of basis matrix $\mathbb{A}$. Recall in the correct ordering $\mathbb{A}$ is a block-diagonal matrix, with diagonal blocks given by $\widetilde{A} ^{[r]}$ for each layer. The maximum singular value of a block-diagonal matrix is given by the maximum over maximum singular values of each block. Thus, we will proceed by considering $\|\widetilde{A}^{[r]}\|$ for arbitrary $P,Q$ and fixed $r$. 



From Corollary \ref{cor:uni}, for uniform knots $\widetilde{A}^{[r]}$ is upper triangular Toeplitz. Appealing to Toeplitz matrix theory, asymptotically (in $n$) tight bounds can be placed on the maximum singular value of $\widetilde{A}^{[r]}$ (and thus equivalently the maximum eigenvalue of $\widetilde{A}^{[r]}(\widetilde{A}^{[r]})^T$) by way of considering the operator's generator function. 
Let $\alpha_\ell$ denote the Toeplitz coefficient for the $i$th diagonal of a Toeplitz matrix, where $\alpha_0$ is the diagonal, $\alpha_{-1}$ the first subdiagonal, and so on. Then the Toeplitz matrix corresponds with a Fourier generator function,
\begin{align*}
{F}(x) & = \sum_{\ell=-\infty}^\infty \alpha_\ell e^{-\mathrm{i}\ell x}.
\end{align*}
The following theorem introduces a specific result from the field of block-Toeplitz operator theory.
\begin{theorem}[Maximum singular value of block-Toeplitz operators \cite{tilli1998singular}]\label{th:toeplitz2}
Let $T_N(F)$ be an $N\times N$ block-Toeplitz matrix, with continuous generating function $F(x): [0,2\pi]\to\mathbb{C}^{m\times m}$.
Then, the maximum singular value is bounded above by
\begin{align*}
    \sigma_{max}(T_N(F)) \leq \max_{x\in[0,2\pi]} \sigma_{max}(F(x)),
\end{align*}
 for all $N\in\mathbb{N}$.
\end{theorem}
In our case we have scalar (i.e. blocksize $N=1$) generating function of $\widetilde{A}^{[r]}$ given by
\begin{equation}
F_r(x) = r \sum_{\ell=0}^r \frac{(-1)^{\ell} }{(\ell)!\,(r-\ell)!} \, e^{-\mathrm{i}\ell x} = r \cdot \frac{(1-e^{\mathrm{i}x})^r}{r!}
= n^{r-1}\frac{(1-e^{\mathrm{i}x})^r}{(r-1)!}.
\end{equation}
Thus we have for $r\in\mathbb{Z}^+$,
\begin{equation}
    \|\widetilde{A}^{[r]}\| \leq \max_x |F_r(x)| = \frac{2^r}{(r-1)!} \leq 4
    \quad\text{for }r\geq1,
\end{equation}
which completes the proof.
\end{proof}

\section{Practical Considerations}
\subsection{Domain Mismatch of Consecutive Layers}
\label{app:mismatch}

If there are splines that are unsupported by data, i.e., there are no data points that lie in the support of the splines, that spline always evaluates to zero and therefore does not contribute in any way to the loss. As a result, the gradient of that spline's weights with respect to the loss will necessarily be zero. The arising issue is a variant of the \emph{dead weight} or \emph{dying ReLU} problem, as described in \citep{clevert2015fast, evci2018detecting} and elsewhere. 

This issue was considered early in KAN literature, and many solutions have arisen, including adaptive spline knot schemes \citep{liu2024kan, liu2024kan2} and batch / layer normalization \citep{ioffe2015batch, ba2016layer}. 

Many adaptive spline knot schemes weight grid points based on the image of the previous layer; for example, \citep{liu2024kan, liu2024kan2} both use an elastic deformation that yields a new set of spline knots by interpolating between the current set of knots and a uniform partition of the histogram of outputs of the previous layer. Such a scheme is empirically driven and is difficult to interpret in terms of traditional machine learning methods or mathematical underpinnings. In contrast, the basis equivalence in Lemma \ref{lemma:basis} suggests that using trainable biases in a multichannel ReLU network is equivalent to free-knot spline approximation. Using this idea, and methods presented in \citep{actor2022polynomial}, we implement KANs with free-knot splines, whose knots are trainable alongside the weights at each layer; in Section \ref{sec:experiments} we show that such KANs outperform traditional KANs with fixed spline grids. This begins to address the question of dependence of the overall accuracy on the spline grids that was voiced in \citep{howard2024multifidelity, abueidda2024deepokan}.

As for normalization schemes, the most popular normalization schemes are imperfect and do not guarantee exact coverage that the image of one layer is the domain of activations at the next layer. If batch normalization of layer inputs is used to map the inputs to a normal $N(0,1)$ distribution, assuming a normal distribution of inputs, we guarantee it is likely that our data will fall in the splines' domain -- 99.7\% of values will lie in the interval $[-3,3]$ -- but it is comparatively unlikely that a specific input will fall in the support of the outermost splines. If we consider a spline supported on an interval $[t_i, t_{i+1}]$, the batch size necessary to make the probability of at least one point in the batch to fall in this interval to be larger than a prescribed value $\tau$ is given by the relationship
\begin{equation}
    \text{batch size} \ge \frac{\log\left( 1 - \tau\right)}{\log \left( \mathbb{P}\left( x \notin [t_i, t_{i+1}] \right) \right)}.
\end{equation}
Thus, if we have a KAN with $n=5$ on a uniform grid of spline knots (so that the set of knots $T=\{-3, -2, -1, 0, 1, 2, 3\}$) and we would like to have at least one training point fall in the support of the first and last splines with probability $\tau = 0.999$, we require a batch size of at least approximately 320, which might be prohibitively large depending on model size, the provided hardware, and the task at hand.


Alternatively, one can construct a uniform normalization, instead of normalizing to a mean-zero, unit-variance distribution, by computing the maximum and minimum of the values of each batch and applying the affine transformation to map the interval into a fixed domain, e.g. $[-1,1]$, which supports the splines. 

Our regression experiments in Section \ref{sec:experiments} employ this uniform normalization, while the Burger's Equation results use batch normalization with a grid interval of $[-4,4]$.

\subsection{Linear Dependence of Spline Superpositions}
\label{app:nullspace}
At each layer, there is a nullspace that arises in KAN bases formed by B-splines (and the equivalent ReLU basis). This linear dependence is clearly illustrated through the following example. Suppose we aim to reconstruct the constant function $f(x_1, x_2) = 1$ with a KAN with \emph{no hidden layers}, using $r$-order splines with fixed knots $T$. Using the spline basis $B_S$, we aim to find weights $V \in \mathbb{R}^{(n+r-1) \times 2}$ so that the reconstruction
\begin{equation*} \begin{split}
\widehat{f}(x_1, x_2) &= \sum_{i=1-r}^{n-1} \sum_{p=1}^2 V_{pi} b_i(x_p) \\
&= \sum_{i=1-r}^{n-1} V_{1,i} b_i(x_1) + \sum_{i=1-r}^{n-1} V_{2,i} b_i(x_2)
\end{split} \end{equation*}
minimizes the prescribed loss function.
By definition, B-splines obey a partition of unity property, i.e. $\forall x_p \in [a,b]$, 
\begin{equation*}
\sum_{i=1-r}^{n-1} b_i(x_p) = 1
\end{equation*}
Thus, for any $\alpha \in \mathbb{R}$, setting $V_{1,i} = \alpha$  and $V_{2,i} = 1-\alpha$ for all $i$, all yield a perfect reconstruction of the target constant function $f$.

This example illustrates the linear dependence inherent to the KAN spline basis, failing to represent constants uniquely \emph{at every layer}. Since the basis of each layer is linearly dependent, the Hessian of the objective function with respect to the weights is necessarily rank-deficient, slowing down convergence when approaching a local minimum \citep{nocedal1999numerical}. As a result, regularization strategies such as weight decay are \emph{necessary} for KANs to train well.

\subsection{KANs in Other Bases}
\label{sec:other-bases}
A similar expansion in a fixed basis can be explicitly written for other choices of KAN bases, namely wavelet bases \citep{bozorgasl2024wav, shukla2024comprehensive} and Chebyshev polynomial bases \citep{ss2024chebyshev, guo2024physics, shukla2024comprehensive}. In particular, when using Chebyshev polynomials as a basis for KANs, the use of function compositions does not enrich the approximation space: the composition of polynomial functions remains a polynomial, albeit of a higher degree. This aligns with early theory on the Kolmogorov Superposition Theorem, which unambiguously states that polynomials used in compositional structures fail to yield an MLP structure with universal approximation properties \citep{vitushkin1964some, vitushkin2004hilbert, cybenko1989approximation}. If one were to turn to Chebyshev polynomials for a KAN basis, one must also contend with a decrease in accuracy when using a higher-order polynomial basis, as noted in \citep{ss2024chebyshev}; this effect is possibly due to ill-conditioning when using only single-precision floating point numbers with high-degree polynomials - an increasingly common limitation on GPU hardware.

\subsection{Computational Cost}
\label{app:computational-cost}
The insight connecting splines to ReLU networks yields an algorithmic construction of a spline basis that replaces the Cox-De Boor formula with ReLU activations raised to the desired spline power. The Cox-De Boor formula has been observed to be computationally expensive \citep{qiu2024relu}, and this reformulation provides an impressive speedup to the underlying computational graph. While a forward-pass through the standard B-spline basis takes $O(PQ(nr+r^2))$ floating-point operations per layer, the ReLU-based formulation requires only order $O(PQ(n+r))$ operations, resulting in a speedup by a factor equal to the spline degree. This removes the need to implement different versions of the spline activations as in \citep{qiu2024relu, so2024higher}. 

The change-of-basis operation, i.e., multiplication by $A$, requires only $O(nr)$ operations, since $A$ is a banded matrix with $r+1$ nonzero diagonals (see Appendix \ref{app:ntk_proof} for explicit construction), and is therefore negligible next to the cost of the contraction against the learnable weights. On a set of predetermined spline knots, the matrix $A$ can additionally be implemented via a convolution stencil. For trainable knots presented later, the entries of $A$ must then be assembled after each gradient update, although there is an additional cost incurred due to the need to backpropagate gradients during training of the spline knot parameters.

\section{Computing Details and Reproducibility}
\label{app:hyperparameters}
For our examples, we run our models with a fixed seed. We use machines with Nvidia A100 GPUs for our code, which expedited training time. Code for PINN simulation relied on hyperparameter selection from the \texttt{jaxpi} library \citep{wang2023expert, wang2024piratenets, wang2025gradient}, but otherwise relied on our own implementation.

Details about network architectures are provided in the main text; details about the normalization to match the image of a spline layer to lie in the domain of the next layer are provided in Section \ref{app:mismatch}.
For all regression examples we utilize an input dataset
with 20000 uniformly randomly generated points on the domain [0.0001,0.9999].
All regression output data are normalized by an affine transformation to the interval [0,1]. The PINNs problems use space-time collocation points on a $64\times 64$ grid and use the entire batch each step. 
The random number generator seed for the dataset is 0, and the batch size is all 20000 data points.
The Adam optimizer uses a learning rate of $10^{-3}$ and LBFGS uses a learning rate of 1.0; LBFGS uses a gradient tolerance of $10^{-12}$; otherwise, all optimizer parameters are left at their default settings. For the PINN problems, we us weight-regularized Adam and additionally employ a exponentially-cyclic learning rate scheduler ranging form $0.0001$ to $0.001$ with a 100-step cycle and a decay parameter $\gamma = 0.9995$, as done in other places, e.g. \citep{wang2023expert}.

For the nonsmooth regression problem, the random number generator seeds for model generation are fixed at 1234, and the ensemble of five runs utilizes seeds 1234 through 1238; the xor regression problem uses seeds 1232 through 1236; for the PINN example, seeds starting at 1234 and increase by 10 for each new run.

\section{Convergence History}
\label{app:convergence}
Two selections of convergence histories for the function regression examples under the same amount of training work are shown in \ref{fig:convergence-history} to illustrate the impact of refinement on training history. The points where the multilevel model loss stagnates before abruptly dropping correspond to the epochs where the model is refined.
\begin{figure}[htbp!]
    \centering
    \includegraphics[width=0.7\textwidth]{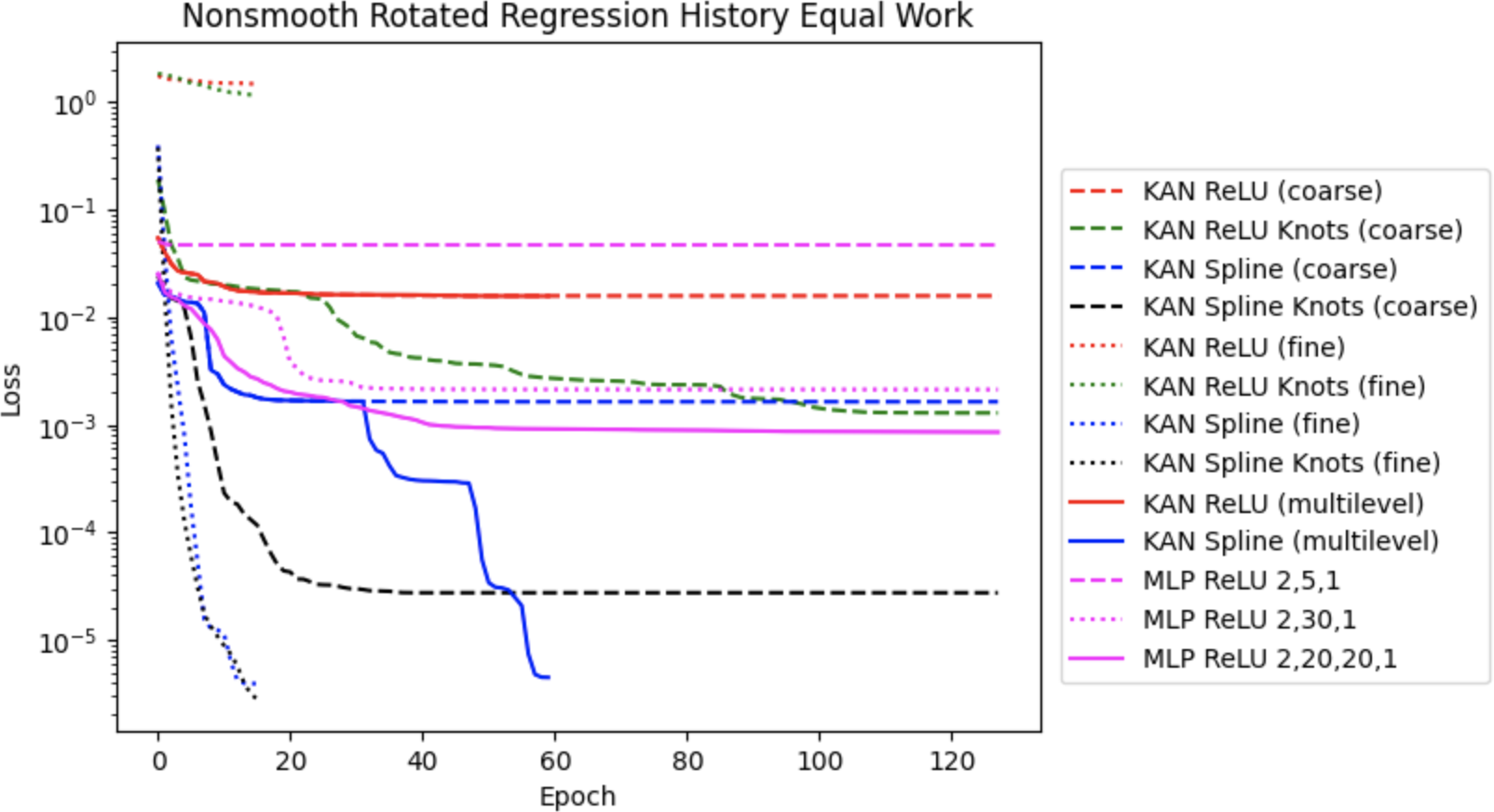}
    \includegraphics[width=0.7\textwidth]{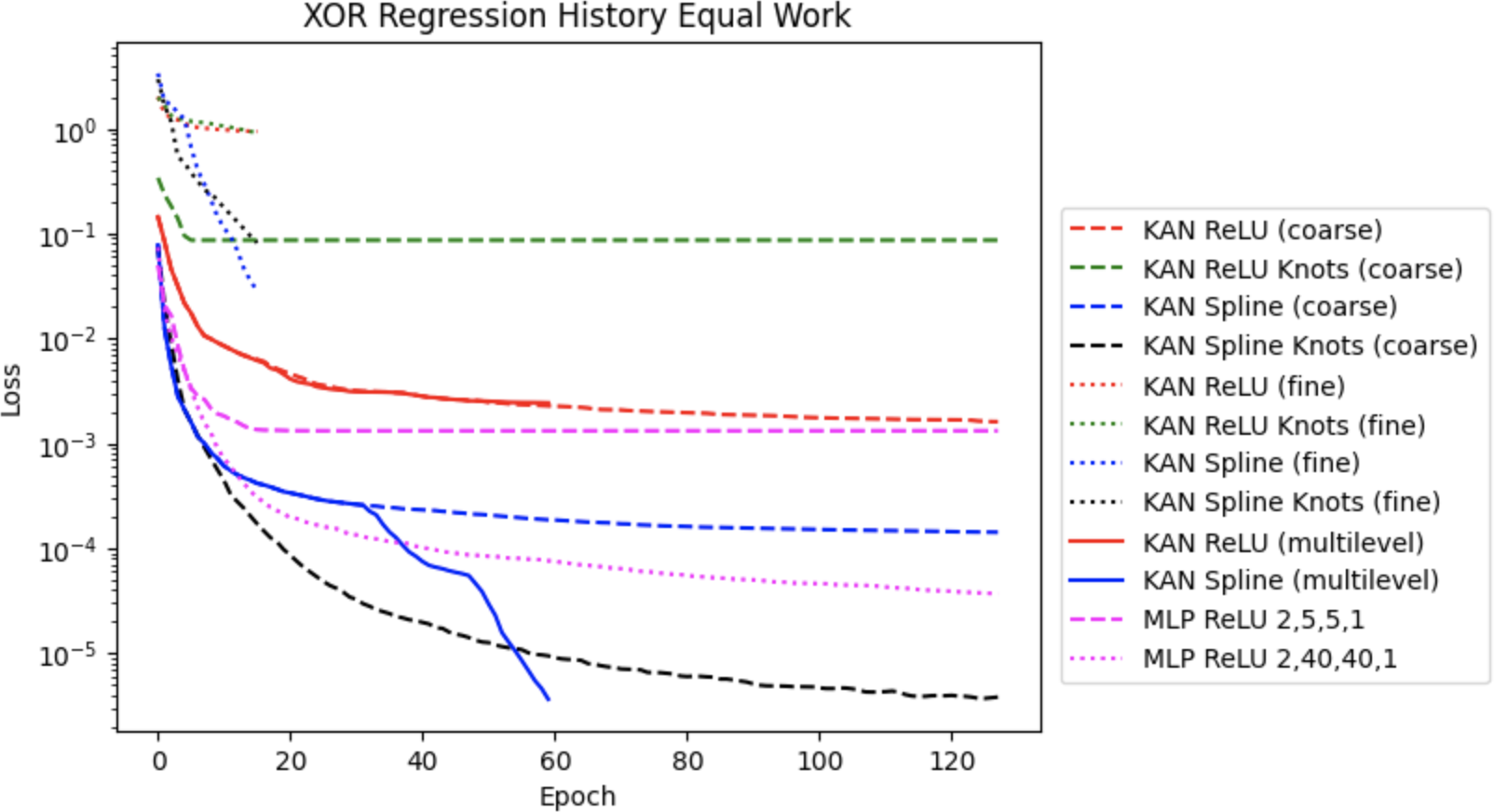}
    \caption{Select convergence history for the nonsmooth and XOR regression examples under approximately same amount of work for all different models.}
    \label{fig:convergence-history}
\end{figure}

\end{document}